\documentclass{article} 
\usepackage{iclr2016_conference,times}
\usepackage{url}

\usepackage{amsmath}
\usepackage{graphicx,subfigure}
\usepackage[colorlinks,linkcolor=red,citecolor=blue,urlcolor=blue,draft]{hyperref}
\usepackage[flushleft]{threeparttable}

\usepackage{algorithm}
\usepackage{algorithmic}

\usepackage{etoolbox}

\usepackage{tabularx}
\usepackage{mdwlist}

\usepackage{multirow}
\usepackage{rotating}
\input{Definitions}

\title{BlackOut: Speeding up Recurrent Neural Network Language Models with Very Large Vocabularies}

\author{Shihao Ji \\
Parallel Computing Lab, Intel\\
\texttt{shihao.ji@intel.com} \\
\And
S. V. N. Vishwanathan \\
Univ. of California, Santa Cruz \\
\texttt{vishy@ucsc.edu}
\And
Nadathur Satish, Michael J. Anderson \& Pradeep Dubey \\
Parallel Computing Lab, Intel\\
\texttt{\{nadathur.rajagopalan.satish,michael.j.anderson,pradeep.dubey\}@intel.com}
}

\iclrfinalcopy 

\begin{document}

\maketitle

\begin{abstract}
  We propose \emph{BlackOut}, an approximation algorithm to efficiently
  train massive recurrent neural network language models (RNNLMs) with 
  million word vocabularies. BlackOut is motivated by using a
  discriminative loss, and we describe a weighted sampling strategy which
  significantly reduces computation while improving stability, sample
  efficiency, and rate of convergence. One way to understand BlackOut
  is to view it as an extension of the DropOut strategy to the output
  layer, wherein we use a discriminative training loss and a weighted
  sampling scheme. We also establish close connections between BlackOut,
  importance sampling, and noise contrastive estimation (NCE).  Our
  experiments, on the recently released one billion word language
  modeling benchmark, demonstrate scalability and accuracy of BlackOut;
  we outperform the state-of-the art, and achieve the lowest perplexity
  scores on this dataset. Moreover, unlike other established methods which typically require GPUs or CPU clusters, we show that a carefully implemented version of BlackOut requires only 1-10 days on a single
  machine to train a RNNLM with a million word vocabulary and billions
  of parameters on one billion words. Although we describe BlackOut in the context of RNNLM training, it can be used to any networks with large softmax output layers.
\end{abstract}

\section{Introduction}
\label{sec:Introduction}

Statistical language models are a crucial component of speech
recognition, machine translation and information retrieval systems. In
order to handle the data sparsity problem associated with traditional
$n$-gram language models (LMs), neural network language models (NNLMs)
\citep{BenDucVin01} represent the history context in a
continuous vector space that can be learned towards error rate
reduction by sharing data among similar contexts. Instead of using fixed 
number of words to represent context, recurrent neural network language 
models (RNNLMs) \citep{MikKarBur10} use a recurrent hidden layer to 
represent longer and variable length histories. RNNLMs significantly 
outperform traditional $n$-gram LMs, and are therefore becoming an 
increasingly popular choice for practitioners \citep{MikKarBur10,SunOpaGau13,DevZbiHua14}. 

\begin{figure*}[htb]\vspace{-0.0cm}
  \begin{center}
    \includegraphics[width=4in]{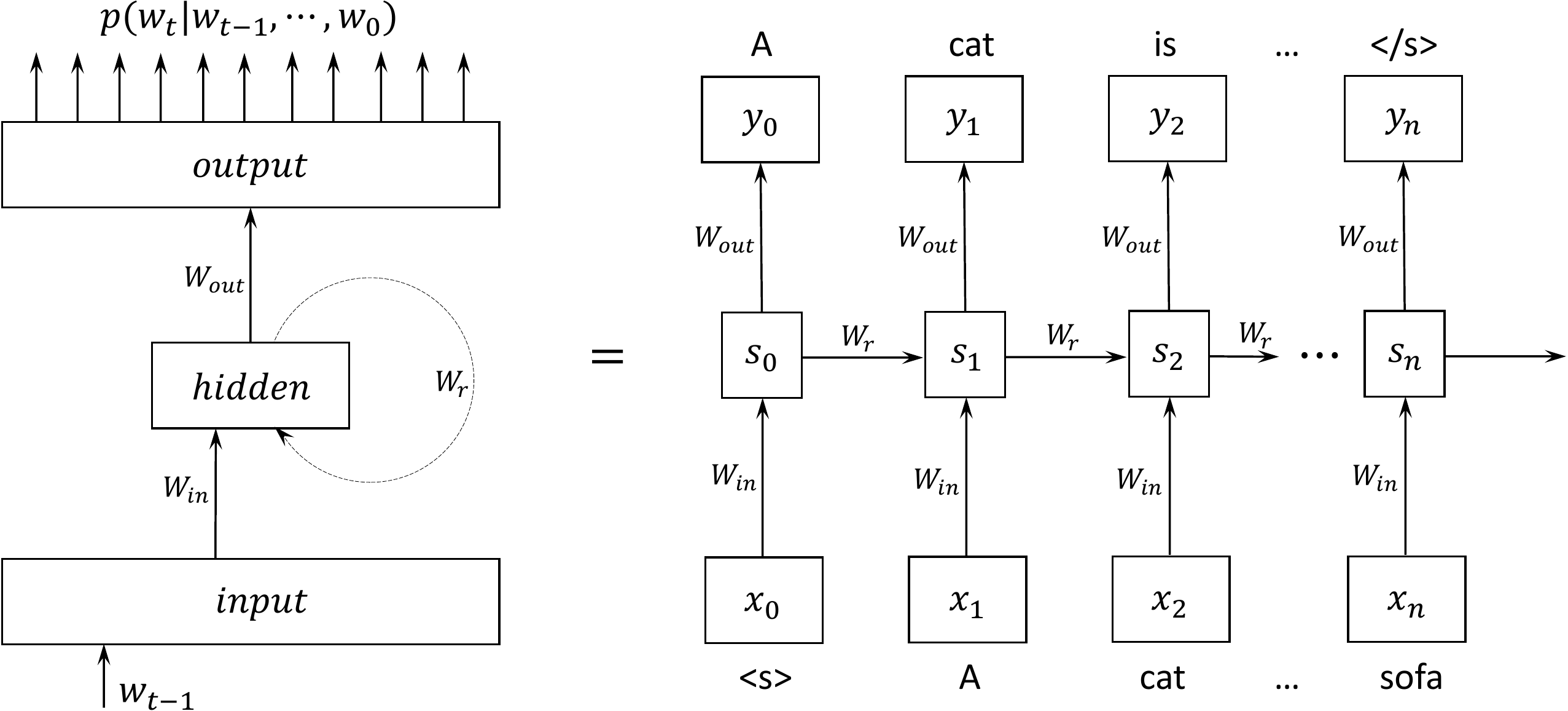}
  \end{center}
  \vspace{-0.0cm}
  \caption{The network architecture of a standard RNNLM and its unrolled
    version for an example input sentence: \texttt{<s>} A cat is sitting
    on a sofa \texttt{</s>}.} 
  \label{fig:rnnlm_arch}
  \vspace{-0.0cm}
\end{figure*}

Consider a standard RNNLM, depicted in Figure~\ref{fig:rnnlm_arch}. The
network has an input layer $x$, a hidden layer $s$ (also called context
layer or state) with a recurrent connection to itself, and an output
layer $y$. Typically, at time step $t$ the network is fed as input
$x_{t} \in \RR^{V}$, where $V$ denotes the vocabulary size, and
$s_{t-1} \in \RR^{h}$, the previous state. It produces a hidden state
$s_{t} \in \RR^{h}$, where $h$ is the size of the hidden layer, which in
turn is transformed to the output $y_{t} \in \RR^{V}$. Different layers
are fully connected, with the weight matrices denoted by
$\Omega=\{W_{in}^{V\times h}, W_r^{h\times h},W_{out}^{V\times h}\}$.

For language modeling applications, the input $x_t$ is a sparse
vector of a 1-of-$V$ (or one-hot) encoding with the element
corresponding to the input word $w_{t-1}$ being 1 and the rest of 
components of $x_{t}$ set to 0; the state of the network $s_t$ is a dense vector,
summarizing the history context $\{w_{t-1},\cdots,w_0\}$ preceding the
output word $w_t$; and the output $y_t$ is a dense vector, with
the $i$-th element denoting the probability of the next word being
$w_{i}$, that is, $p(w_i|w_{t-1},\cdots,w_0)$, or more concisely,
$p(w_i|s_t)$. The input to output transformation occurs via:
\begin{align}
  \label{eq:rnnlm-hidden}
  s_t &= \sigma(W_{in}^Tx_t+W_rs_{t-1})\\
  \label{eq:rnnlm}
  y_t &= f(W_{out}s_t),
\end{align}
where $\sigma(v)=1/(1+\exp(-v))$ is the sigmoid activation function, and
$f(\cdot)$ is the softmax function
$f(u_i):=\exp(u_i)/\sum_{j=1}^{V}\exp(u_j)$.

One can immediately see that if $x_{t}$ uses a 1-of-$V$ encoding, then
the computations in equation \eqref{eq:rnnlm-hidden} are relatively inexpensive
(typically $h$ is of the order of a few thousand, and the computations
are $\Ocal(h^2)$), while the computations in equation \eqref{eq:rnnlm} are
expensive (typically $V$ is of the order of a million, and the
computations are $\Ocal(Vh)$). Similarly, back propagating the gradients
from the output layer to the hidden layer is expensive. Consequently,
the training times for some of the largest models reported in literature
are of the order of weeks \citep{MikDeoPov11,WilPraMrv15}.

In this paper, we ask the following question: Can we design an
approximate training scheme for RNNLM which will improve on the state of
the art models, while using significantly less computational resources?
Towards this end, we propose \emph{BlackOut} an approximation algorithm
to efficiently train massive RNNLMs with million word vocabularies. 
BlackOut is motivated by using a discriminative loss, and we describe a 
weighted sampling strategy which significantly reduces computation while 
improving stability, sample efficiency, and rate of convergence. 
We also establish close connections between BlackOut, importance sampling, 
and noise contrastive estimation (NCE) \citep{GutHyv12,MniTeh12}, and demonstrate that BlackOut mitigates some of the limitations of both previous methods. Our experiments, on the recently released one billion word language
modeling benchmark \citep{CheMikSch14}, demonstrate scalability and
accuracy of BlackOut; we outperform the state-of-the art, achieving
the lowest perplexity scores on this dataset. Moreover, unlike other
established methods which typically require GPUs or CPU clusters, we
show that a carefully implemented version of BlackOut requires only 1-10
days on a single CPU machine to train a RNNLM with a million word vocabulary
and billions of parameters on one billion words. 

One way to understand BlackOut is to view it as an extension of the 
DropOut strategy \citep{SriHinKri14} to the output layer, wherein we 
use a discriminative training loss and a weighted sampling scheme. The connection to DropOut is mainly from the way they operate in model training and model evaluation. Similar to DropOut, in BlackOut training a subset of output layer is sampled and trained at each training batch and when evaluating, the full network participates. Also, like DropOut, a regularization technique, our experiments show that the models trained by BlackOut are less prone to overfitting. A primary difference between them is that DropOut is routinely used at input and/or hidden layers of deep neural networks, while BlackOut only operates at output layer. We chose the name BlackOut in light of the similarities between our method and DropOut, and the complementary they offer to train deep neural networks.


\section{BlackOut: A sampling-based approximation} 
\label{sec:blackout}

We will primarily focus on estimation of the matrix $W_{out}$. To
simplify notation, in the sequel we will use $\theta$ to denote
$W_{out}$ and $\theta_{j}$ to denote the $j$-th row of
$W_{out}$. Moreover, let $\inner{\cdot}{\cdot}$ denote the dot product
between two vectors. Given these notations, one can rewrite equation
\eqref{eq:rnnlm} as 
\begin{align}
  p_\theta\rbr{w_i|s} =
  \frac{\exp\rbr{\inner{\theta_i}{s}}}{\sum_{j=1}^V\exp\rbr{\inner{\theta_j}{s}}}\quad
  \forall i\in\{1,\cdots,V\}.
  \label{eq:full_softmax}
\end{align}
RNNLMs with a softmax output layer are typically trained using cross-entropy
as the loss function, which is equivalent to maximum likelihood (ML)
estimation, that is, to find the model parameter $\theta$ which
maximizes the log-likelihood of target word $w_i$, given a history context $s$:
\begin{align}
  J_{ml}^s(\theta) = \log p_\theta(w_i|s), 
  \label{eq:ml_loss}
\end{align}
whose gradient is given by
\begin{align}
  \frac{\partial J_{ml}^s(\theta)}{\partial\theta} &= \frac{\partial}{\partial\theta}\inner{\theta_i}{s} - \sum_{j=1}^Vp_\theta\rbr{w_j|s}\frac{\partial}{\partial\theta}\inner{\theta_j}{s}, \nonumber\\
  &= \frac{\partial}{\partial\theta}\inner{\theta_i}{s} - \EE_{p_\theta(w|s)}\sbr{\frac{\partial}{\partial\theta}\inner{\theta_w}{s}}.
  \label{eq:gradient_of_ml_loss}
\end{align}
The gradient of log-likelihood is expensive to evaluate because (1) the cost of computing $p_\theta(w_j|s)$ is $\Ocal(Vh)$ and (2) the summation above takes time linear in the vocabulary size $\Ocal(V)$.

To alleviate the computational bottleneck of computing the gradient \eqref{eq:gradient_of_ml_loss}, we propose to use the following
\emph{discriminative} objective function for training RNNLM: 
\begin{align}
  J_{disc}^s(\theta) = \log \tilde{p}_\theta(w_i|s) + \sum_{j \in
  S_K}\log(1-\tilde{p}_\theta(w_j|s)), 
  \label{eq:disc_loss}
\end{align}
where $S_{K}$ is a set of indices of $K$ words drawn from the vocabulary, and $i\notin S_K$. Typically, 
$K$ is a tiny fraction of $V$, and in our experiments we use $K \approx V/200$. To generate $S_{K}$ we
will sample $K$ words from the vocabulary using an easy to sample distribution
$Q(w)$, and set $q_j : =\frac{1}{Q(w_j)}$ in order to compute
\begin{align}
  \tilde{p}_\theta(w_i|s) =
  \frac{q_i\exp\rbr{\inner{\theta_i}{s}}}{q_i\exp\rbr{\inner{\theta_i}{s}}+\sum_{j\in
  S_K}q_j\exp\rbr{\inner{\theta_j}{s}}}. 
  \label{eq:posterior}
\end{align}

Equation~\ref{eq:disc_loss} is the cost function of a standard logistic
regression classifier that discriminates one positive sample $w_i$ from
$K$ negative samples $w_j, \forall j\in S_K$. The first term in
\eqref{eq:disc_loss} corresponds to the traditional maximum likelihood
training, and the second term explicitly pushes down the probability of
negative samples in addition to the implicit shrinkage enforced by the
denominator of \eqref{eq:posterior}. In our experiments, we found the
discriminative training~(\ref{eq:disc_loss}) outperforms the maximum
likelihood training (the first term of Eq.~\ref{eq:disc_loss}) in all the 
cases, with varying degree of accuracy improvement depending on $K$. 

The weighted softmax function~(\ref{eq:posterior}) can be considered as a stochastic version of the standard softmax~(\ref{eq:full_softmax}) on a different base measure. While the standard softmax~(\ref{eq:full_softmax}) uses a base measure which gives equal weights to all words, and has support over the entire vocabulary, the base measure used in (\ref{eq:posterior}) has support only on $K+1$ words: the target word $w_i$ and $K$ samples from $Q(w)$. The noise portion of (\ref{eq:posterior}) has the motivation from the sampling scheme, and the $q_i$ term for target word $w_i$ is introduced mainly to balance the contributions from target word and noisy sample words.\footnote{It's shown empirically in our experiments that setting $q_i=1$ in~(\ref{eq:posterior}) hurts the accuracy significantly.} Other justifications are discussed in Sec.~\ref{sec:IS} and Sec.~\ref{sec:NCE}, where we establish close connections between BlackOut, importance sampling, and noise contrastive estimation.

Due to the weighted sampling property of BlackOut, some words might be sampled multiple times according to the proposal distribution $Q(w)$, and thus their indices may appear multiple times in $S_K$. As $w_i$ is the target word, which is assumed to be included in computing (\ref{eq:posterior}), we therefore set $i\notin S_K$ explicitly.

Substituting \eqref{eq:posterior} into \eqref{eq:disc_loss} and letting
$u_j = \inner{\theta_j}{s}$ and
$\tilde{p}_j = \tilde{p}_\theta\rbr{w_j|s}$, we have
\begin{align}
  J_{disc}^s(\theta) \propto u_i - (K+1)\log\mkern-15mu\sum_{k\in
  \cbr{i}\cup S_K}\mkern-15muq_k\exp(u_k) + \sum_{j\in
  S_K}\mkern-5mu\log\rbr{\sum_{k\in\cbr{i} \cup S_K}\mkern-15mu q_k\exp(u_k)-q_j\exp(u_j)}.
\end{align}
Then taking derivatives with respect to
$u_j, \forall j\in\cbr{i} \cup S_K$, yields
\begin{align}
  \frac{\partial J_{disc}^s(\theta)}{\partial u_i} &= 1 - \rbr{K + 1 - \sum_{j\in S_K}\frac{1}{1-\tilde{p}_j}}\tilde{p}_i \label{eq:gradient_disc_i}\\
  \frac{\partial J_{disc}^s(\theta)}{\partial u_j} &= - \rbr{K + 1 -
                                                     \sum_{k\in S_K\setminus \cbr{j}}\frac{1}{1-\tilde{p}_k}}\tilde{p}_j, \quad \text{for } j\in S_K. \label{eq:gradient_disc_j}
\end{align}
By the chain rule of derivatives, we can propagate the errors backward
to previous layers and compute the gradients with respect to the full
model parameters $\Omega$. In contrast to Eq.~\ref{eq:gradient_of_ml_loss}, Eqs.~\ref{eq:gradient_disc_i} and 
\ref{eq:gradient_disc_j} are much cheaper to evaluate as (1) the cost of 
computing $\tilde{p}_j$ is $\Ocal(Kh)$ and (2) the summation takes $\Ocal(K)$,
hence roughly a $V/K$ times of speed-up. 

Next we turn our attention to the proposal distribution $Q(w)$. In the
past, a uniform distribution or the unigram distribution have been
advocated as promising candidates for sampling distributions
\citep{BenSen03, JeaChoMem15, BenSen08, MniTeh12}. As we will see in the
experiments, neither one is suitable for a wide range of
datasets, and we find that the power-raised unigram distribution of \cite{MikSutCheCoretal13} is very important in this context:
\begin{align}
  Q_\alpha(w) \propto p_{uni}^\alpha(w),\quad \alpha\in[0, 1]. 
  \label{eq:proposal_function}
\end{align}
Note that $Q_\alpha(w)$ is a generalization of uniform distribution
(when $\alpha=0$) and unigram distribution (when $\alpha=1$).  The
rationale behind our choice is that by tuning $\alpha$, one can
interpolate smoothly between sampling popular words, as advocated by the
unigram distribution, and sampling all words equally. The best $\alpha$ 
is typically dataset and/or problem dependent; in our experiments, we use
a holdout set to find the best value of $\alpha$. It's worth noting that this sampling strategy has been used by \cite{MikSutCheCoretal13} 
in a similar context of word embedding, while here we explore its effect in the language modeling applications.

After BlackOut training, we evaluate the predictive performance of 
RNNLM by perplexity. To calculate perplexity, we explicitly normalize 
the output distribution by using the exact softmax function~(\ref{eq:full_softmax}). This is similar to DropOut \citep{SriHinKri14}, 
wherein a subset of network is sampled and trained at each training batch and 
when evaluating, the full network participates.

\subsection{Connection to Importance Sampling}\label{sec:IS}
BlackOut has a close connection to importance sampling (IS). To see this, differentiating the logarithm of Eq.~\ref{eq:posterior} with respect to model parameter $\theta$, we have
\begin{align}
    \frac{\partial}{\partial\theta}\log\tilde{p}_\theta(w_i|s)&= \frac{\partial}{\partial\theta}\inner{\theta_i}{s} - \frac{1}{\sum_{k\in\{i\}\cup S_K}q_k\exp(\inner{\theta_k}{s})}\sum_{j\in\{i\}\cup S_K}q_j\exp(\inner{\theta_j}{s})\frac{\partial}{\partial\theta}\inner{\theta_j}{s} \nonumber \\
    &= \frac{\partial}{\partial\theta}\inner{\theta_i}{s} - \EE_{\tilde{p}_\theta(w|s)}\sbr{\frac{\partial}{\partial\theta}\inner{\theta_w}{s}}.
\end{align}
In contrast with Eq.~\ref{eq:gradient_of_ml_loss}, it shows that the weighted softmax function~(\ref{eq:posterior}) corresponds to an IS-based estimator of the standard softmax~(\ref{eq:full_softmax}) with a proposal distribution $Q(w)$.

Importance sampling has been applied to NNLMs with large output layers in previous works \citep{BenSen03, BenSen08, JeaChoMem15}. However, either uniform distribution or unigram distribution is used for sampling and all aforementioned works exploit the maximum likelihood learning of model parameter $\theta$. By contrast, BlackOut uses a discriminative training~(\ref{eq:disc_loss}) and a power-raised unigram distribution $Q_\alpha(w)$ for sampling; these two changes are important to mitigate some of limitations of IS-based approaches. While an IS-based approach with a uniform proposal distribution is very stable for training, it suffers from large bias due to the apparent divergence of the uniform distribution from the true data distribution $p_\theta(w|s)$. On the other hand, a unigram-based IS estimate can make learning unstable due to the high variance \citep{BenSen03, BenSen08}. Using a power-raised unigram distribution $Q_\alpha(w)$ entails a better trade-off between bias and variance, and thus strikes a better balance between these two extremes. In addition, as we will see from the experiments, the discriminative training of BlackOut speeds up the rate of convergence over the traditional maximum likelihood learning.

\subsection{Connection to Noise Contrastive Estimation}\label{sec:NCE}
The basic idea of NCE is to transform the density estimation problem to the problem of learning by comparison, e.g., estimating the parameters of a binary classifier that distinguishes samples from the data distribution $p_d$ from samples generated by a known noise distribution $p_n$ \citep{GutHyv12}. In the language modeling setting, the data distribution $p_d$ will be the distribution $p_\theta(w|s)$ of interest, and the noise distribution $p_n$ is often chosen from the ones that are easy to sample from and possibly close to the true data distribution (so that the classification problem isn't trivial). While \cite{MniTeh12} uses a context-independent (unigram) noise distribution $p_n(w)$, BlackOut can be formulated into the NCE framework by considering a context-dependent noise distribution $p_n(w|s)$, estimated from $K$ samples drawn from $Q(w)$, by
\begin{align}
  p_n(w_i|s) = \frac{1}{K}\sum_{j\in S_K}\frac{q_j}{q_i}p_\theta(w_j|s),
  \label{eq:noise distribution}
\end{align}
which is a probability distribution function under the expectation that $K$ samples are drawn from $Q(w)$: $S_K\sim Q(w)$ since $\EE_{S_K\sim Q(w)}(p_n(w_i|s))=Q(w_i)$ and $\EE_{S_K\sim Q(w)}(\sum_{i=1}^Vp_n(w_i|s))=1$ (See the proof in Appendix~\ref{app:noise dist}).

Similar to \cite{GutHyv12}, noise samples are assumed $K$ times more frequent than data samples so that data points are generated from a mixture of two distributions: $\frac{1}{K+1}p_\theta(w|s)$ and $\frac{K}{K+1}p_n(w|s)$. Then the conditional probability of sample $w_i$ being generated from the data distribution is
\begin{align}
  p_\theta(D=1|w_i,s) = \frac{p_\theta(w_i|s)}{p_\theta(w_i|s) + Kp_n(w_i|s)}.
  \label{eq:posterior2}
\end{align}
Inserting Eq.~\ref{eq:noise distribution} into Eq.~\ref{eq:posterior2}, we have
\begin{align}
  p_\theta(D=1|w_i,s) = \frac{q_i\exp(\inner{\theta_i}{s})}{q_i\exp(\inner{\theta_i}{s}) + \sum_{j\in S_K}q_j\exp(\inner{\theta_j}{s})},
  \label{eq:posterior3}
\end{align}
which is exactly the weighted softmax function defined in~(\ref{eq:posterior}). Note that due to the noise distribution proposed in Eq.~\ref{eq:noise distribution}, the expensive denominator (or the partition function $Z$) of $p_\theta(w_j|s)$ is canceled out, while in \cite{MniTeh12} the partition function $Z$ is either treated as a free parameter to be learned or approximated by a constant. \cite{MniTeh12} recommended to set $Z=1.0$ in the NCE training. However, from our experiments, setting $Z=1.0$ often leads to sub-optimal solutions\footnote{Similarly, \cite{CheLiuGal15} reported that setting $\ln(Z)=9$ gave them the best results.} and different settings of $Z$ sometimes incur numerical instability since the log-sum-exp trick\footnote{https://en.wikipedia.org/wiki/LogSumExp} can not be used there to shift the scores of the output layer to a range that is amenable to the exponential function. BlackOut does not have this hyper-parameter to tune and the log-sum-exp trick still works for the weighted softmax function~(\ref{eq:posterior}). Due to the discriminative training of NCE and BlackOut, they share the same objective function~(\ref{eq:disc_loss}).

We shall emphasize that according to the theory of NCE, the $K$ samples should be sampled from the noise distribution $p_n(w|s)$. But in order to calculate $p_n(w|s)$, we need the $K$ samples drawn from $Q(w)$ beforehand. As an approximation, we use the same $K$ samples drawn from $Q(w)$ as the $K$ samples from $p_n(w|s)$, and only use the expression of $p_n(w|s)$ in (\ref{eq:noise distribution}) to evaluate the noise density value required by Eq. \ref{eq:posterior2}. This approximation is accurate since $\EE_{S_K\sim Q(w)}(p_n(w_i|s))=Q(w_i)$ as proved in Appendix~\ref{app:noise dist}, and we find empirically that it performs much better (with improved stability) than using a unigram noise distribution as in \cite{MniTeh12}.

\subsection{Related Work}

Many approaches have been proposed to address the difficulty of training deep neural networks with large output spaces. In general, they can be categorized into four categories: 
\begin{itemize}
\item \textit{Hierarchical softmax} \citep{MorBen05,MniHin08} uses a hierarchical binary tree representation of the output layer with the $V$ words as its leaves. It allows exponentially faster computation of word probabilities and their gradients, but the predictive performance of the resulting model is heavily dependent on the tree used, which is often constructed heuristically. Moreover, by relaxing the constraint of a binary structure, \cite{LeOpaAll11} introduces a structured output layer with an arbitrary tree structure constructed from word clustering. All these methods speed up both the model training and evaluation considerably.
\item \textit{Sampling-based approximations} select at random or heuristically a small subset of the output layer and estimate gradient only from those samples. The use of importance sampling in \cite{BenSen03, BenSen08, JeaChoMem15}, and the use of NCE \citep{GutHyv12} in \cite{MniTeh12} all fall under this category, so does the more recent use of Locality Sensitive Hashing (LSH) techniques \citep{ShrPin14, VijShlMon14} to select a subset of good samples. BlackOut, with close connections to importance sampling and NCE, also falls in this category. All these approaches only speed up the model training, while the model evaluation still remains computationally challenging.
\item \textit{Self normalization} \citep{DevZbiHua14} extends the cross-entropy loss function by explicitly encouraging the partition function of softmax to be as close to 1.0 as possible. Initially, this approach only speeds up the model evaluation and more recently it's extended to facilitate the training as well with some theoretical guarantees \citep{AndKle14, AndRabKle15}.
\item \textit{Exact gradient on limited loss functions} \citep{VinBreBou15} introduces an algorithmic approach to efficiently compute the exact loss, gradient update for the output weights in $\Ocal(h^2)$ per training example instead of $\Ocal(Vh)$. Unfortunately, it only applies to a limited family of loss functions that includes squared error and spherical softmax, while the standard softmax isn't included.
\end{itemize}

As discussed in the introduction, BlackOut also shares some similarity to DropOut \citep{SriHinKri14}. While DropOut is often applied to input and/or hidden layers of deep neural networks to avoid feature co-adaptation and overfitting by uniform sampling, BlackOut applies to a softmax output layer, uses a weighted sampling, and employs a discriminative training loss. We chose the name \emph{BlackOut} in light of the similarities between our method and DropOut, and the complementary they offer to train deep neural networks.

\section{Implementation and Further Speed-up} 
\label{sec:implementation}

We implemented BlackOut on a standard machine with a dual-socket 28-core
Intel\textsuperscript\textregistered Xeon\textsuperscript\textregistered \footnote{\scriptsize Intel and Xeon are trademarks of Intel Corporation in the U.S. and/or other countries.} 
Haswell CPU. To achieve high throughput, we train RNNLM with
Back-Propagation Through Time (BPTT) \citep{RumHinWil88} with
mini-batches \citep{CheWanLiu14}. We use RMSProp \citep{Hinton12} for
learning rate scheduling and gradient clipping \citep{BenBouPas13} to
avoid the gradient explosion issue of recurrent networks. We use the
latest Intel MKL library (version 11.3.0) for SGEMM calls, which has
improved support for tall-skinny matrix-matrix multiplications, which
consume about 80\% of the run-time of RNNLMs.

It is expensive to access and update large models with billions of
parameters. Fortunately, due to the 1-of-$V$ encoding at input layer and
the BlackOut sampling at output layer, the model update on $W_{in}$ and
$W_{out}$ is sparse, i.e., only the model parameters corresponding to
input/output words and the samples in $S_K$ are updated at each training 
batch. However, subnet updates have to be done carefully due to the
dependency within RMSProp updating procedure. We therefore propose an
approximated RMSProp that enables an efficient subnet update and thus
speeds up the algorithm even further. Details can be found in
Appendix~\ref{app:subnet update}.

\section{Experiments} 
\label{sec:exp}

In our experiments, we first compare BlackOut, NCE and exact softmax
(without any approximation) using a small dataset. We then evaluate the
performance of BlackOut on the recently released one billion word
language modeling benchmark \citep{CheMikSch14} with a vocabulary size
of up to one million. We compare the performance of BlackOut on a standard 
CPU machine versus the state-of-the-arts reported in the literature that 
are achieved on GPUs or on clusters of CPU nodes. Our
implementation and scripts are open sourced at \url{https://github.com/IntelLabs/rnnlm}.

\paragraph{Corpus}

Models are trained and evaluated on two different corpora: a small
dataset provided by the RNNLM Toolkit\footnote{http://www.rnnlm.org/}, 
and the recently released one billion word language modeling
benchmark\footnote{https://code.google.com/p/1-billion-word-language-modeling-benchmark/},
which is perhaps the largest public dataset in language modeling. The
small dataset has 10,000 training sentences, with 71,350 words in total
and 3,720 unique words; and the test perplexity is evaluated on 1,000
test sentences. The one billion word benchmark was constructed from a
monolingual/English corpora; after all necessary preprocessing including
de-duplication, normalization and tokenization, 30,301,028 sentences
(about 0.8 billion words) are randomly selected for training, 6,075
sentences are randomly selected for test and the remaining 300,613
sentences are reserved for future development and can be used as holdout
set.

\subsection{Results On Small Dataset}

We evaluate BlackOut, NCE and exact softmax (without any approximation)
on the small dataset described above. This small dataset is used so that
we can train the standard RNNLM algorithm with exact softmax within a
reasonable time frame and hence to provide a baseline of expected
perplexity. There are many other techniques involved in the training,
such as RMSProp for learning rate scheduling \citep{Hinton12}, subnet
update (Appendix~\ref{app:subnet update}), and mini-batch splicing
\citep{CheWanLiu14}, etc., which can affect the perplexity
significantly. For a fair comparison, we use the same tricks and settings for all the algorithms, and only evaluate the impact of the different approximations (or no approximation) on the softmax output layer.  Moreover, there are a few hyper-parameters that have strong impact on the predictive
performance, including $\alpha$ of the proposal distribution
$Q_\alpha(w)$ for BlackOut and NCE, and additionally for NCE, the
partition function $Z$. We pay an equal amount of effort to tune these
hyper-parameters for BlackOut and NCE on the validation set as number of
samples increases.

Figure \ref{fig:small_samples} shows the perplexity reduction as a
function of number of samples $K$ under two different vocabulary settings:
(a) a full vocabulary of 3,720 words, and (b) using the most frequent
2,065 words as vocabulary. The latter is a common approach used in
practice to accelerate RNNLM computation by using RNNLM to predict only
the most frequent words and handling the rest using an $n$-gram model
\citep{SchGau05}. We will see similar vocabulary settings when we
evaluate BlackOut on the large scale one billion word benchmark.

\begin{figure*}[htb]\vspace{-0.0cm}
  \begin{center}
    \subfigure{\label{fig:small_full}\includegraphics[width=2.35in]{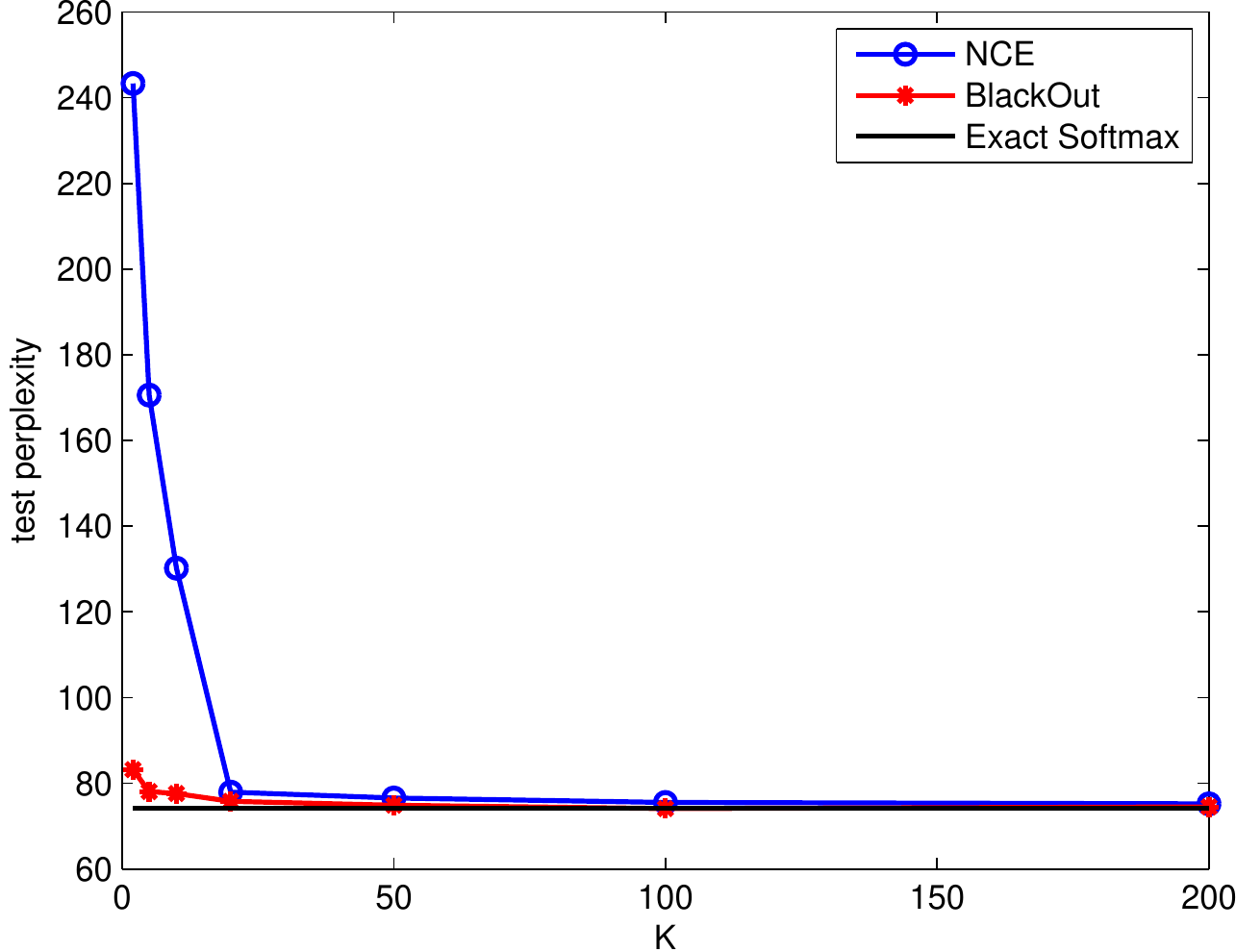}}\hspace{0.3in}
    \subfigure{\label{fig:small_trunc}\includegraphics[width=2.3in]{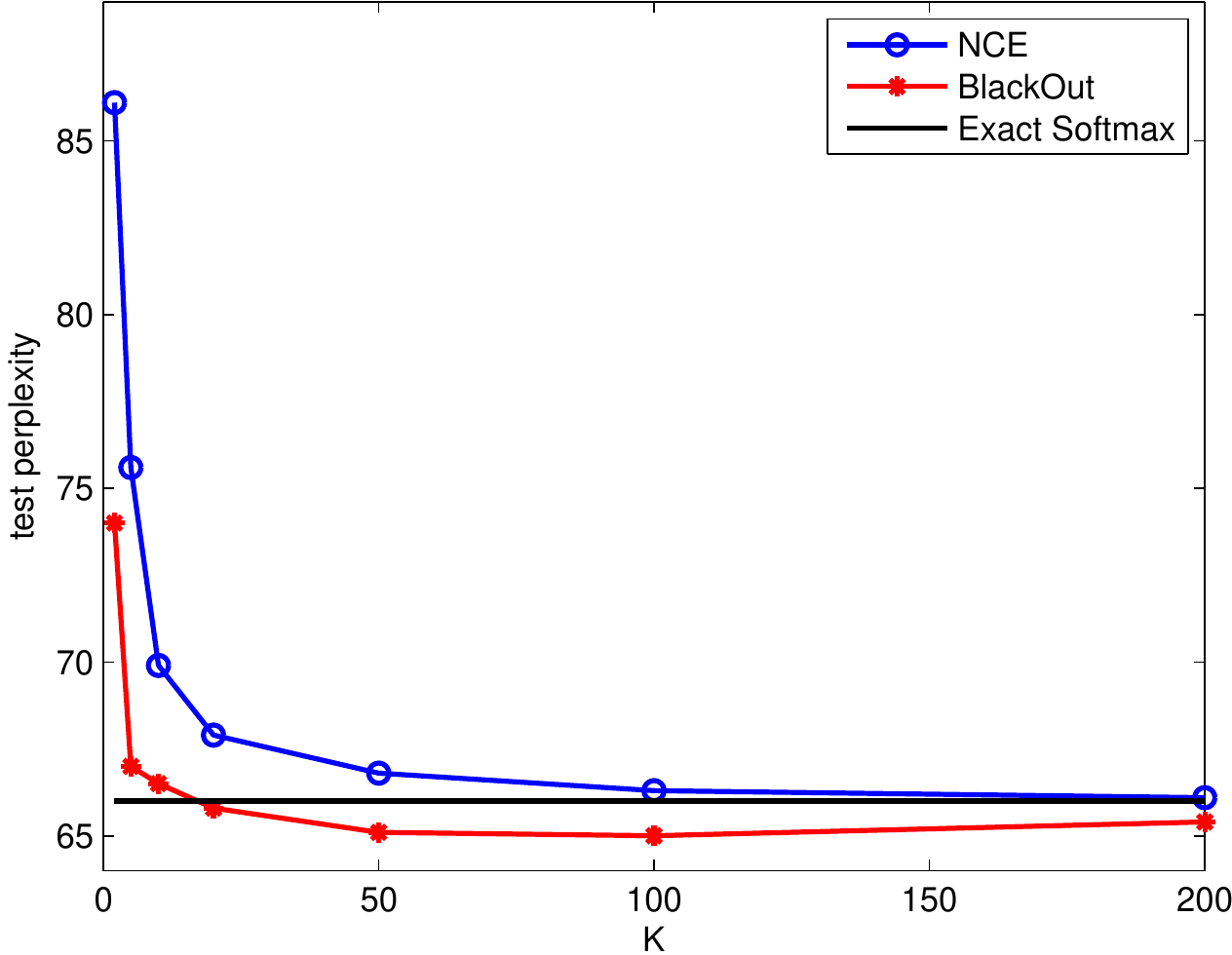}}
  \end{center}\vspace{-0.0cm}
  \caption{Test perplexity evolution as a function of number of samples $K$ (a) with a full vocabulary of 3,720 words, and (b) with the most frequent 2,065 words in vocabulary. The experiments are executed on the RNNLMs with 16 hidden units.}\label{fig:small_samples}\vspace{-0.0cm}
\end{figure*}

As can be seen, when the size of the samples increases, in general both BlackOut and NCE improve their prediction accuracy under the two vocabulary settings, and even with only 2 samples both algorithms still converge to reasonable solutions. BlackOut can utilize samples much more effectively than NCE as manifested by the significantly lower perplexities achieved by BlackOut, especially when number of samples is small; Given about 20-50 samples, BlackOut and NCE reach similar perplexities as the exact softmax, which is expensive to train as it requires to evaluate all the words in the vocabularies. When the vocabulary size is 2,065, BlackOut achieves even better perplexity than that of the exact softmax. This is possible since BlackOut does stochastic sampling at each training example and uses the full softmax output layer in prediction; this is similar to DropOut that is routinely used in input layer and/or hidden layers of deep neural networks \citep{SriHinKri14}. As in DropOut, BlackOut has the benefit of regularization and avoids feature co-adaption and is possibly less prone to overfitting. To verify this hypothesis, we evaluate the perplexities achieved on the training set for different algorithms and provide the results in Figure~\ref{fig:small_samples_train} at Appendix~\ref{app:train perplexy}. As can been seen, the exact softmax indeed overfits to the training set and reaches lower training perplexities than NCE and BlackOut.

Next we compare the convergence rates of BlackOut and NCE when training the RNNLMs with 16 hidden units for a full vocabulary of 3,720 words. Figures \ref{fig:full_lc_10} and \ref{fig:full_lc_50} plot the learning curves of BlackOut and NCE when 10 samples or 50 samples are used in training, respectively. The figure shows that BlackOut enjoys a much faster convergence rate than NCE, especially when number of samples is small (Figure~\ref{fig:full_lc_10}); but this advantage gets smaller when number of samples increases (Figure~\ref{fig:full_lc_50}). We also observed similar behavior when we evaluated BlackOut and NCE on the large scale one billion word benchmark.

\begin{figure*}[htb]\vspace{-0.0cm}
  \begin{center}
    \subfigure{\label{fig:full_lc_10}\includegraphics[width=2.35in]{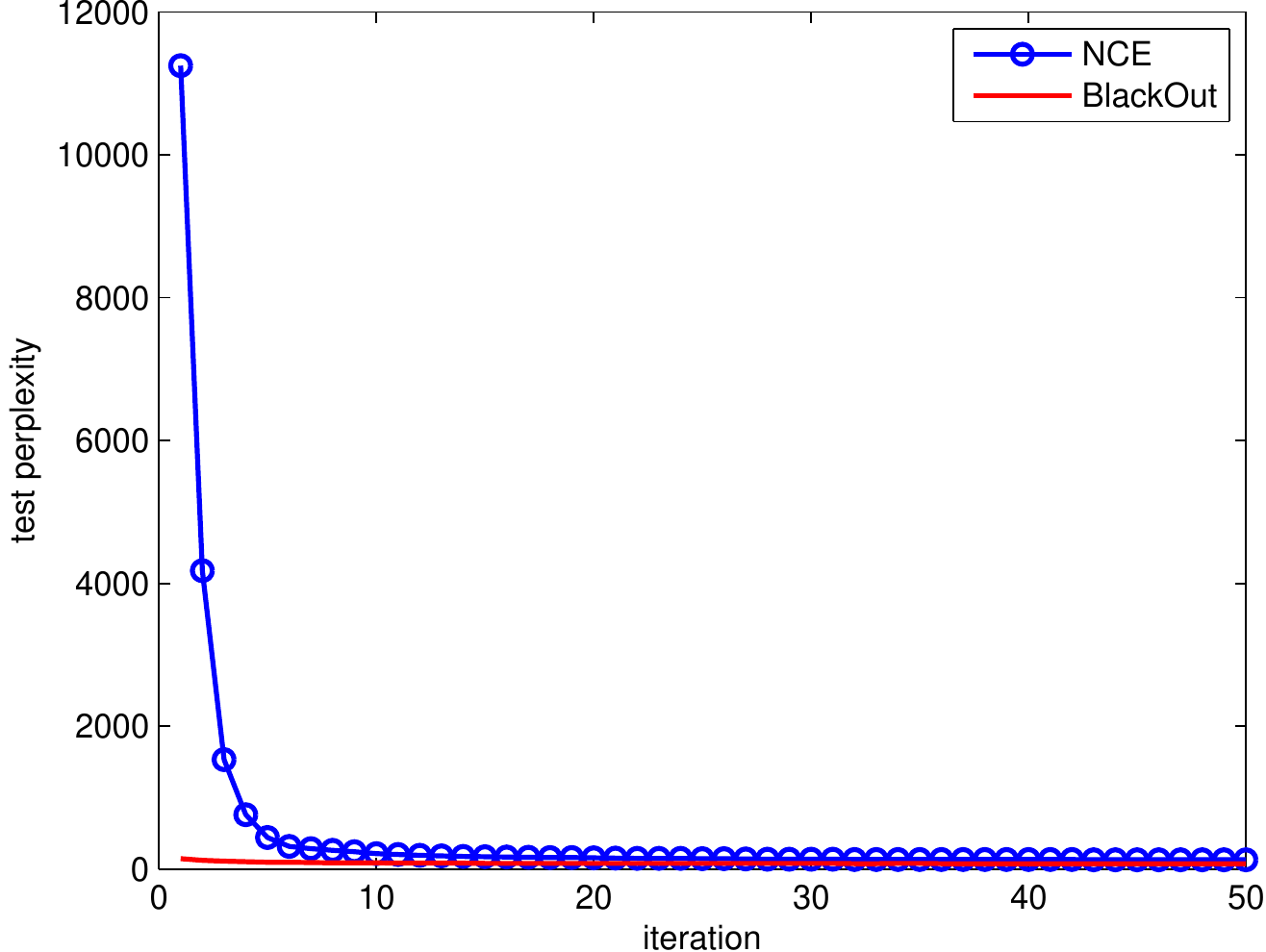}}\hspace{0.3in}
    \subfigure{\label{fig:full_lc_50}\includegraphics[width=2.3in]{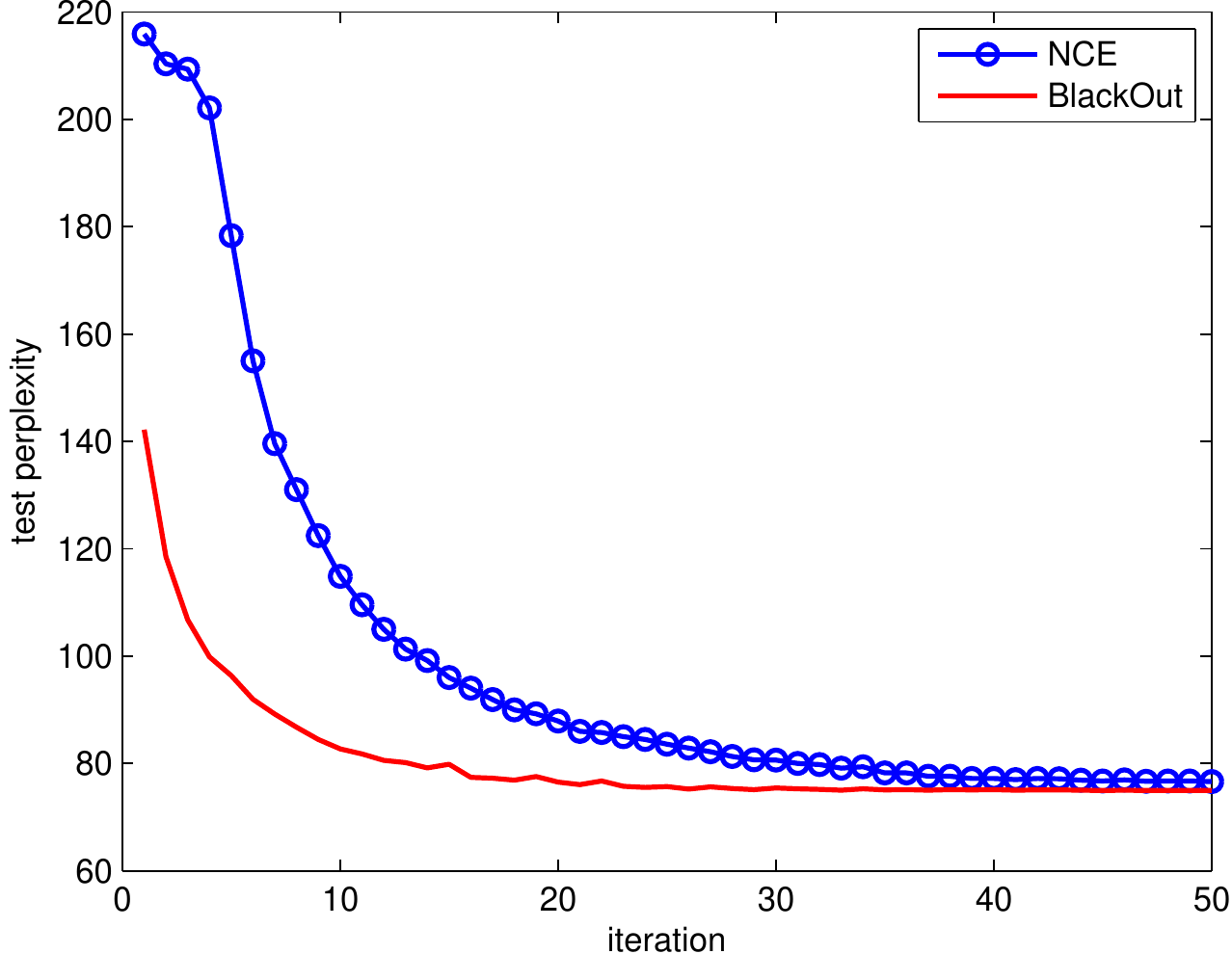}}
  \end{center}\vspace{-0.0cm}
  \caption{The learning curves of BlackOut and NCE when training the RNNLMs with 16 hidden units with (a) 10 samples, and (b) 50 samples.}\label{fig:full_lc}\vspace{-0.0cm}
\end{figure*}

\subsection{Results On One Billion Word Benchmark}
We follow the experiments from \cite{WilPraMrv15} and \cite{LeJaiHin15} and compare the performance of BlackOut with the state-of-the-art results provided by them. While we evaluated BlackOut on a dual-socket 28-core Intel\textsuperscript\textregistered Xeon\textsuperscript\textregistered Haswell machine, \cite{WilPraMrv15} implemented RNNLM with the NCE approximation on NVIDIA GTX Titan GPUs, and \cite{LeJaiHin15} executed an array of recurrent networks, including deep RNN and LSTM, without approximation on a CPU cluster. Besides the time-to-solution comparison, these published results enable us to cross-check the predictive performance of BlackOut with another implementation of NCE or with other competitive network architectures.

\subsubsection{When vocabulary size is 64K}

Following the experiments in \cite{WilPraMrv15}, we evaluate the performance of BlackOut on a vocabulary of 64K most frequent words. This is similar to the scenario in Figure~\ref{fig:small_trunc} where the most frequent words are kept in vocabulary and the rest rare words are mapped to a special \texttt{<unk>} token. We first study the importance of $\alpha$ of the proposal distribution $Q_\alpha(w)$ and the discriminative training~(\ref{eq:disc_loss}) as proposed in BlackOut. As we discussed in Sec.~\ref{sec:blackout}, when $\alpha=0$, the proposal distribution $Q_\alpha(w)$ degenerates to a uniform distribution over all the words in the vocabulary, and when $\alpha=1$, we recover the unigram distribution. Thus, we evaluate the impact of $\alpha$ in the range of $[0, 1]$. Figure~\ref{fig:power} shows the evolution of test perplexity as a function of $\alpha$ for the RNNLMs with 256 hidden units. As can be seen, $\alpha$ has a significant impact on the prediction accuracy. The commonly used uniform distribution (when $\alpha=0$) and unigram distribution (when $\alpha=1$) often yield sub-optimal solutions. For the dataset and experiment considered, $\alpha=0.4$ gives the best perplexity (consistent on holdout set and test set). We therefore use $\alpha=0.4$ in the experiments that follow. The number of samples used is 500, which is about 0.8\% of the vocabulary size. 

\begin{figure*}[htb]\vspace{-0.0cm}
  \begin{center}
    \subfigure{\label{fig:power}\includegraphics[width=2.3in]{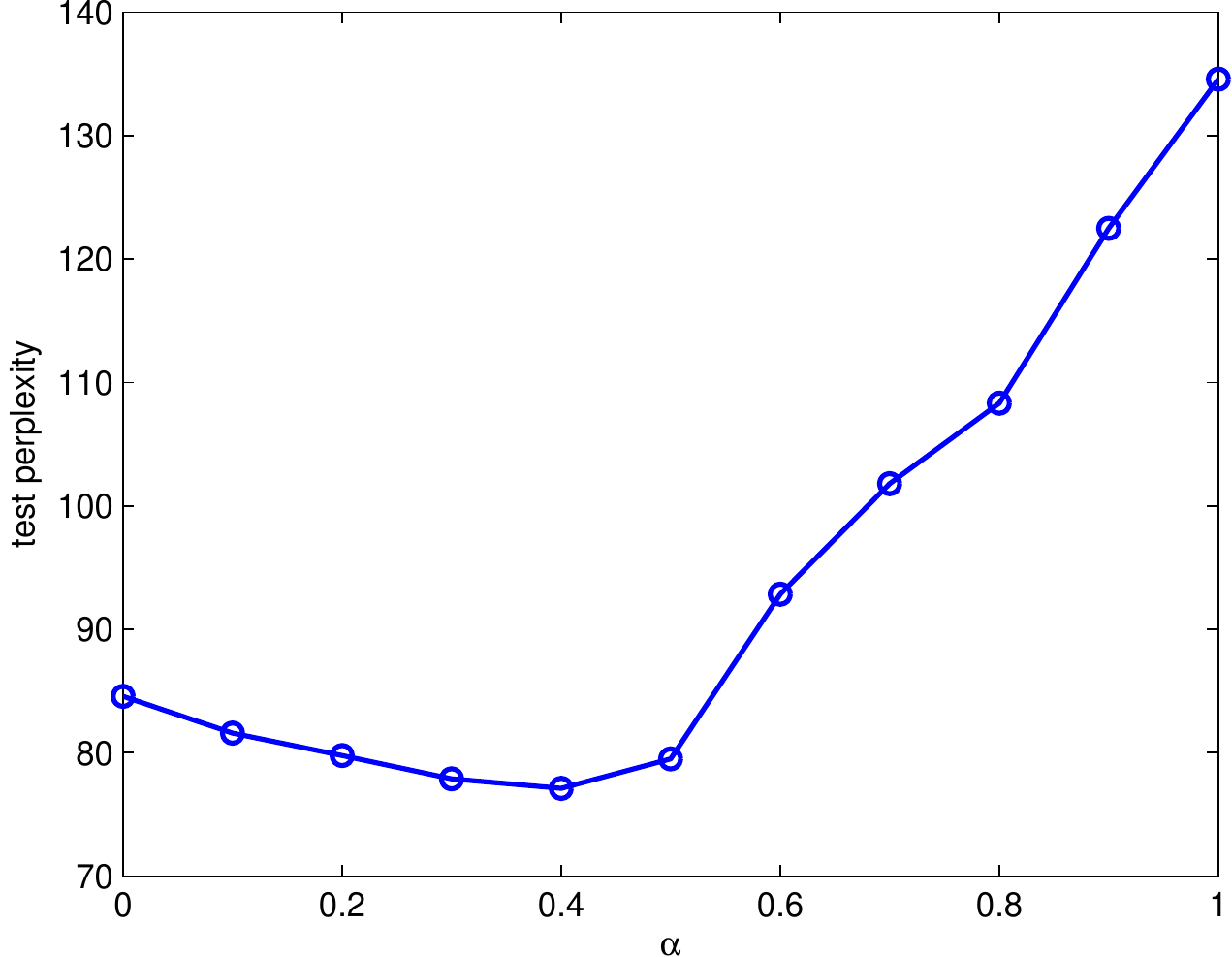}}\hspace{0.3in}
    \subfigure{\label{fig:disc}\includegraphics[width=2.3in]{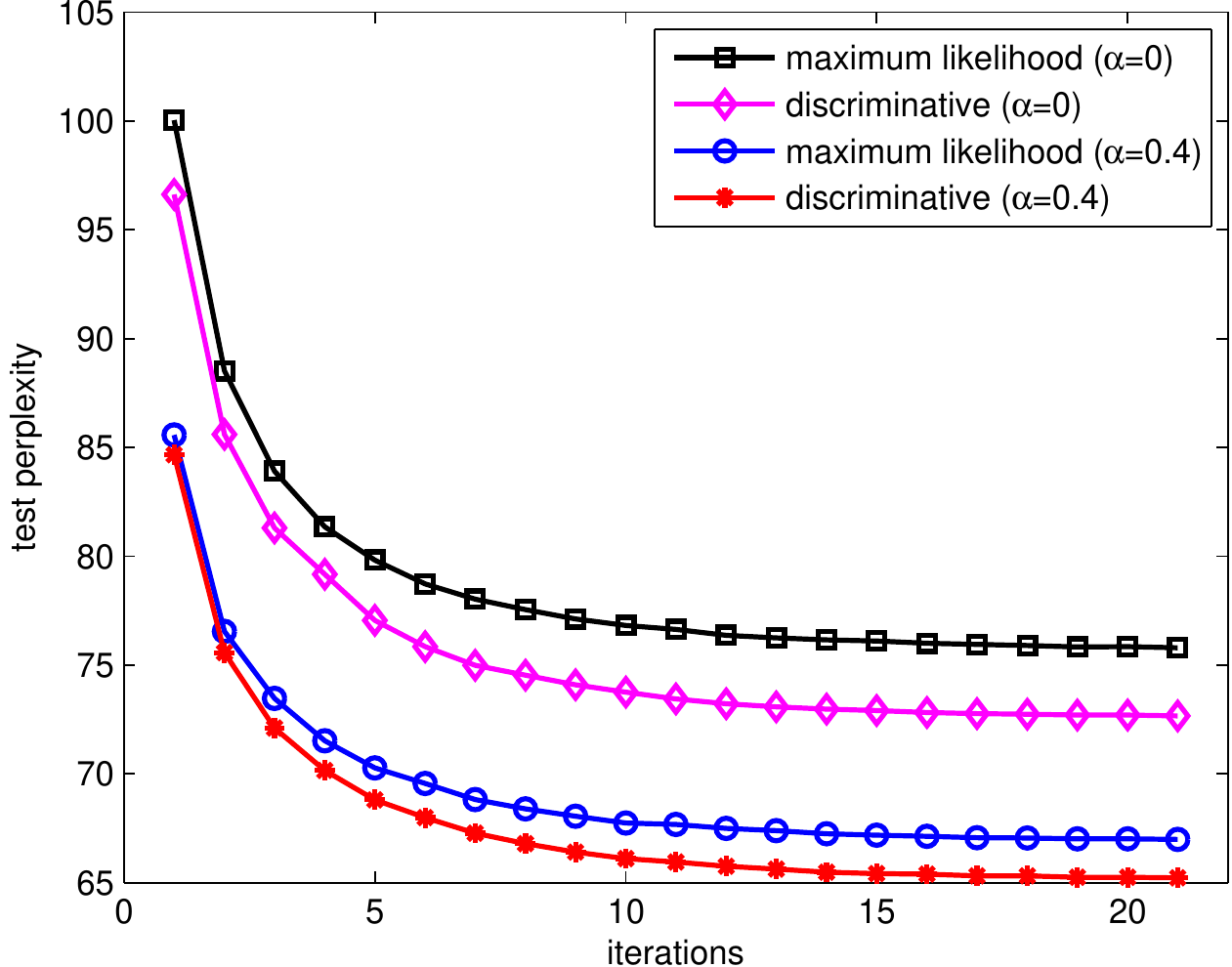}}
  \end{center}\vspace{-0.0cm}
  \caption{(a) The impact of $\alpha$ evaluated when 256 hidden units are used; (b) The learning curves of maximum likelihood and discriminative training when 512 hidden units are used.} \vspace{-0.0cm}
\end{figure*}

Figure~\ref{fig:disc} demonstrates the impact of discriminative training~(\ref{eq:disc_loss}) over the maximum likelihood training (the first term of Eq.~\ref{eq:disc_loss}) on the RNNLMs with 512 hidden units using two different $\alpha$'s. In general, we observe 1-3 points of perplexity reduction due to discriminative training over traditional maximum likelihood training.

Finally, we evaluate the scalability of BlackOut when number of hidden
units increases. As the dataset is large, we observed that the
performance of RNNLM depends on the size of the hidden layer: they
perform better as the size of the hidden layer gets larger. As a
truncated 64K word vocabulary is used, we interpolate the RNNLM scores
with a full size 5-gram to fill in rare word probabilities
\citep{SchGau05,ParLiuGal10}. We report the interpolated perplexities BlackOut achieved and compare them with the results from \cite{WilPraMrv15} in Table~\ref{tab:64k}. As can be
seen, BlackOut reaches lower perplexities than those reported in
\cite{WilPraMrv15} within comparable time frames (often 10\%-40\%
faster). We achieved a perplexity of 42.0 when the hidden layer size is
4096. To the best of our knowledge, this is the lowest perplexity
reported on this benchmark.

\begin{table*}[htb]
\caption {Performance on the one billion word benchmark by interpolating RNNLM on a 64K word vocabulary with a full-size KN 5-gram LM.}
\label{tab:64k}
\begin{center}
\begin{threeparttable}
\begin{tabular}{|l|c|c|c|c|c|}
\hline           & \#Params  & \multicolumn{2}{c|}{Test Perplexity} & \multicolumn{2}{c|}{Time to Solution} \\
\cline{3-6}  \multicolumn{1}{|c|}{Model} & [millions] & Published$^1$ & BlackOut        & Published$^1$ & BlackOut \\
\hline          KN 5-gram      & 1,748 & \multicolumn{2}{c|}{66.95} & \multicolumn{2}{c|}{45m} \\
\hline          RNN-128 + KN 5-gram     & 1,764 & 60.8  & 59.0 & 6h    & 9h  \\
\hline          RNN-256 + KN 5-gram      & 1,781 & 57.3 & 55.1 & 16h   & 14h  \\
\hline          RNN-512 + KN 5-gram      & 1,814 & 53.2 & 51.5 & 1d2h  & 1d  \\
\hline          RNN-1024 + KN 5-gram     & 1,880  & 48.9 & 47.6 & 2d2h  & 1d14h  \\
\hline          RNN-2048 + KN 5-gram     & 2,014 & 45.2  & 43.9 & 4d7h  & 2d15h  \\
\hline          RNN-4096 + KN 5-gram     & 2,289 & \textbf{42.4} & \textbf{42.0}  & 14d5h & 10d  \\
\hline
\end{tabular}
\begin{tablenotes}
  \scriptsize\item $^1$Data from Table 1 of \cite{WilPraMrv15}.
\end{tablenotes}
\end{threeparttable}
\end{center}
\end{table*}

\subsubsection{When vocabulary size is 1M}
In the final set of experiments, we evaluate the performance of BlackOut
with a very large vocabulary of 1,000,000 words, and the results are
provided in Table~\ref{tab:1m}. This is the largest vocabulary used on
this benchmark that we could find in existing literature. We consider
the RNNLM with 1,024 hidden units (about 2 billion parameters) and 2,048
hidden units (about 4.1 billion parameters) and compare their test
perplexities with the results from \cite{LeJaiHin15}. We use 2,000
samples, 0.2\% of the vocabulary size, for BlackOut training with
$\alpha=0.1$. Comparing to the experiments with the 64K word vocabulary,
a much smaller $\alpha$ is used here since the sampling rate (0.2\%) is
much lower than that is used (0.8\%) when the vocabulary size is 64K,
and a smaller $\alpha$ strikes a better balance between sample coverage
per training example and convergence rate. In contrast, NCE with the
same setting converges very slowly (similar to
Figure~\ref{fig:full_lc_10}) and couldn't reach a competitive perplexity
within the time frame considered, and its results are not reported here.

As the standard RNN/LSTM algorithms (without approximation) are used in
\cite{LeJaiHin15}, a cluster of 32 CPU machines (at least 20 cores each)
are used to train the models for about 60 hours. BlackOut enables us to
train this large model using a single CPU machine for 175
hours. Since different model architectures are used in the experiments (deep RNN/LSTM vs. standard RNNLM), the direct comparison of test perplexity isn't very meaningful. However, this experiment demonstrates that even though our
largest model is about 2-3 times larger than the models evaluated in
\cite{LeJaiHin15}, BlackOut, along with a few other optimization
techniques, make this large scale learning problem still feasible on a
single box machine without using GPUs or CPU clusters.

\begin{table*}[htb]
\caption {Performance on the one billion word benchmark with a vocabulary of 1,000,000 words. Single model (RNN/LSTM-only) perplexities are reported; no interpolation is applied to any models.}
\label{tab:1m}
\begin{center}
\begin{tabular}{|c|l|c|}
\hline           & Model & Perplexity \\
\hline  Results from      & LSTM (512 units) & \textbf{68.8}  \\
\cline{2-3}  \cite{LeJaiHin15} & IRNN (4 layers, 512 units) & 69.4  \\
\cline{2-3}  60 hours  & IRNN (1 layer, 1024 units + 512 linear units) & 70.2  \\
\cline{2-3}  32 machines & RNN (4 layers, 512 tanh units) & 71.8 \\
\cline{2-3}  & RNN (1 layer, 1024 tanh units + 512 linear units) & 72.5  \\
\hline          Our Results    & RNN (1 layer, 1024 sigmoid units) & 78.4  \\
\cline{2-3}          175 hours, 1 machine   & RNN (1 layer, 2048 sigmoid units) & \textbf{68.3}  \\
\hline
\end{tabular}
\end{center}
\end{table*}


Last, we collect all the state of the art results we are aware of on this benchmark and summarize them in Table~\ref{tab:stoa}. Since all the models are the interpolated ones, we interpolate our best RNN model\footnote{To be consistent with the benchmark in \cite{CheMikSch14}, we retrained it with the full-size vocabulary of about 0.8M words.} from Table~\ref{tab:1m} with the KN 5-gram model and achieve a perplexity score of 47.3. Again, different papers provide their best models trained with different architectures and vocabulary settings. Hence, an absolutely fair comparison isn't possible. Regardless of these discrepancies, our models, within different groups of vocabulary settings, are very competitive in terms of prediction accuracy and model size.
   
\begin{table*}[htb]
\caption {Comparison with the state of the art results reported on the one billion word benchmark.}
\label{tab:stoa}
\begin{center}
\begin{threeparttable}
\begin{tabular}{|l|c|c|}
\hline       \multicolumn{1}{|c|}{Model}    & \#Params [billions] & Test Perplexity   \\
\hline          RNN-1024 (full vocab) + MaxEnt$^1$ & 20 & 51.3  \\
\hline          RNN-2048 (full vocab) + KN 5-gram$^2$ & 5.0 & 47.3  \\
\hline          RNN-1024 (full vocab) + MaxEnt + 3 models$^1$ & 42.9  & 43.8\\
\hline          RNN-4096 (64K vocab) + KN 5-gram$^3$   & 2.3  & 42.4  \\
\hline          RNN-4096 (64K vocab) + KN 5-gram$^2$   & 2.3  & 42.0  \\
\hline
\end{tabular}
\begin{tablenotes}
  \scriptsize\item $^1$Data from \cite{CheMikSch14}; $^2$Our results; $^3$Data from \cite{WilPraMrv15}.
\end{tablenotes}
\end{threeparttable}
\end{center}
\end{table*}

\section{Conclusion}\label{sec:conclusion}
We proposed \emph{BlackOut}, a sampling-based approximation, to train RNNLMs
with very large vocabularies (e.g., 1 million). We established its
connections to importance sampling and noise contrastive estimation
(NCE), and demonstrated its stability, sample efficiency and rate of convergence on the recently released one billion word language modeling
benchmark. We achieved the lowest reported perplexity on this benchmark without using GPUs or CPU clusters.

As for future extensions, our plans include exploring other proposal
distributions $Q(w)$, and theoretical properties of the generalization
property and sample complexity bounds for BlackOut. We will also
investigate a multi-machine distributed implementation. 

\section*{Acknowledgments}
We would like to thank Oriol Vinyals, Andriy Mnih and the anonymous reviewers for their excellent comments and suggestions, which helped improve the quality of this paper.

\bibliography{bibfile}
\bibliographystyle{iclr2016_conference}

\newpage
\appendix

\begin{center}
\Large{BlackOut: Speeding up Recurrent Neural Network Language Models with Very Large Vocabularies}\\
\large{\color{red}(Supplementary Material)}
\end{center}

\section{Noise distribution $p_n(w_i|s)$} \label{app:noise dist}
\begin{theorem}
The noise distribution function $p_n(w_i|s)$ defined in Eq.~\ref{eq:noise distribution} is a probability distribution function under the expectation that $K$ samples in $S_K$ are drawn from $Q(w)$ randomly, $S_K\sim Q(w)$, such that $\EE_{S_K\sim Q(w)}(p_n(w_i|s))=Q(w_i)$ and $\EE_{S_K\sim Q(w)}(\sum_{i=1}^Vp_n(w_i|s))=1$.
\end{theorem}
\begin{proof}
\begin{align*}
  \EE_{S_K\sim Q(w)}(p_n(w_i|s)) &= \EE_{S_K\sim Q(w)}\rbr{\frac{1}{K}\sum_{j\in S_K}\frac{q_j}{q_i}p_\theta(w_j|s)} \\
  &= \frac{Q(w_i)}{K}\EE_{S_K\sim Q(w)}\rbr{\sum_{j\in S_K}\frac{p_\theta(w_j|s)}{Q(w_j)}} \\
  &= \frac{Q(w_i)}{K}\sum_{w_k, \forall k\in S_K}\rbr{\prod_{k\in S_K}Q(w_k)\cdot\sum_{j\in S_K}\frac{p_\theta(w_j|s)}{Q(w_j)}} \\
  &= \frac{Q(w_i)}{K}K \\
  &= Q(w_i)
\end{align*}
\begin{align*}
  \EE_{S_K\sim Q(w)}\rbr{\sum_{i=1}^Vp_n(w_i|s)}= \sum_{i=1}^V\EE_{S_K\sim Q(w)}\rbr{p_n(w_i|s)}=\sum_{i=1}^V\rbr{Q(w_i)}=1
\end{align*}
\end{proof}

\section{Perplexities on Training set} \label{app:train perplexy}

\begin{figure*}[htb]\vspace{-0.0cm}
  \begin{center}
    \subfigure{\label{fig:small_full_train}\includegraphics[width=2.35in]{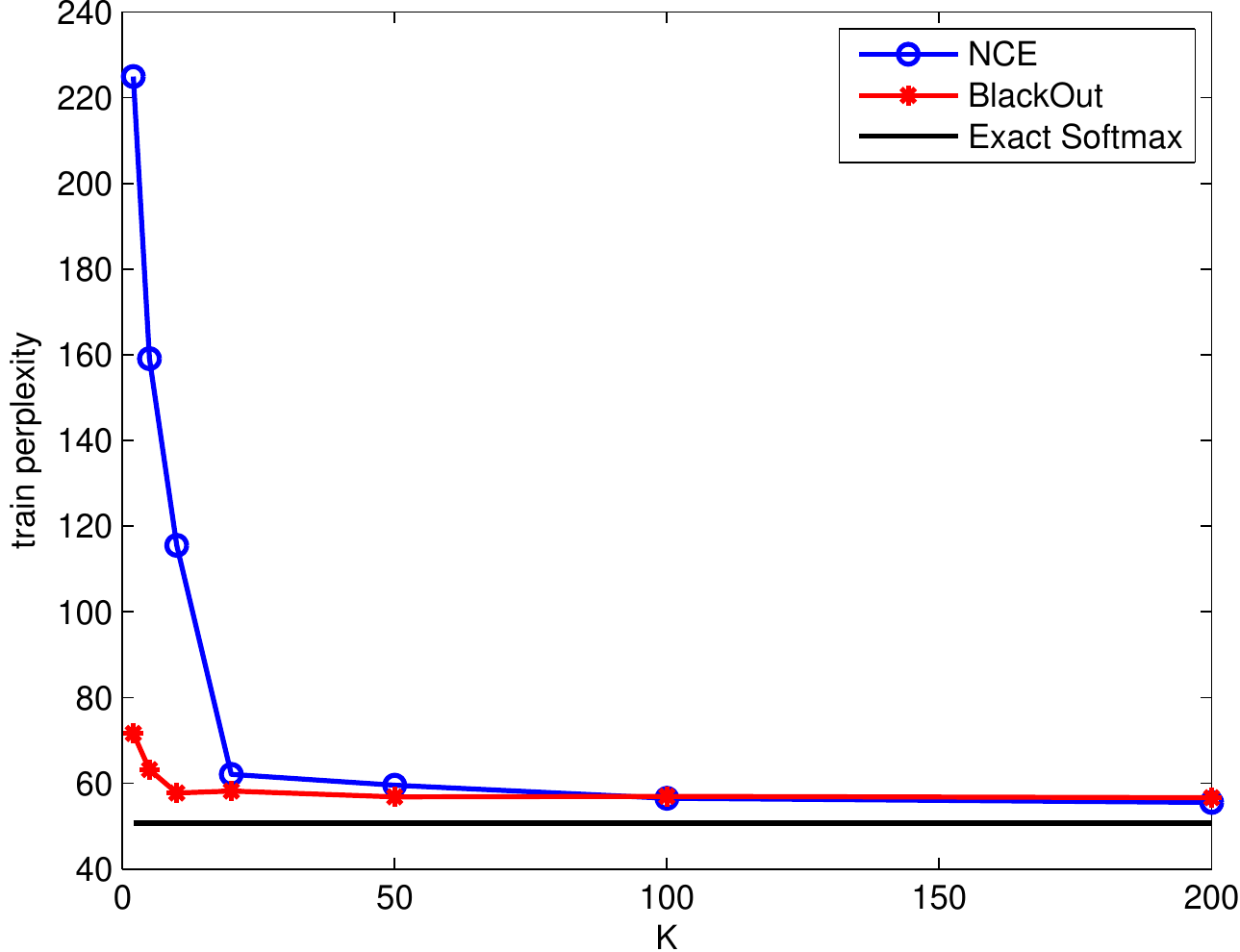}}\hspace{0.3in}
    \subfigure{\label{fig:small_trunc_train}\includegraphics[width=2.3in]{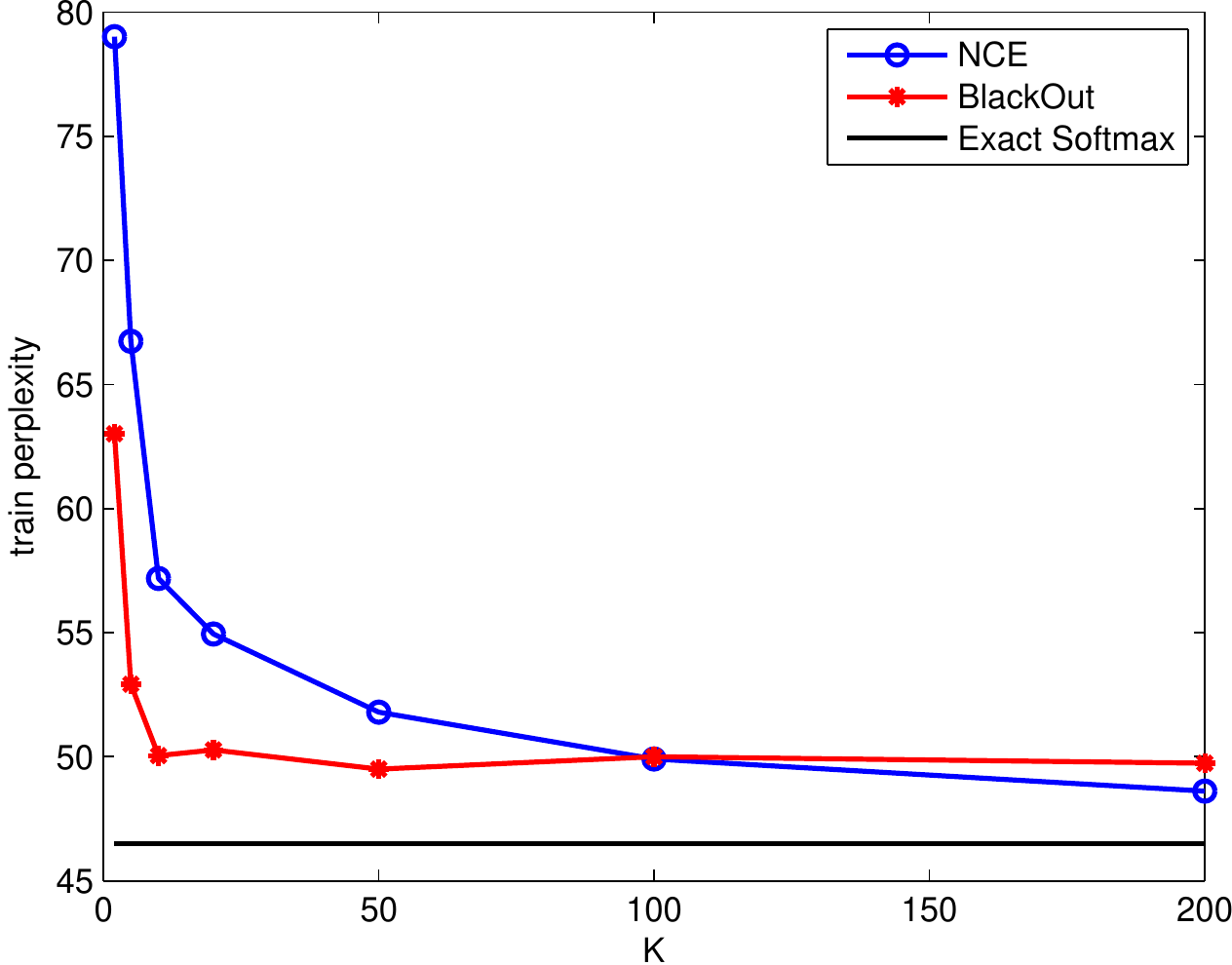}}
  \end{center}\vspace{-0.0cm}
  \caption{Training perplexity evolution as a function of number of samples $K$ (a) with a full vocabulary of 3,720 words, and (b) with the most frequent 2,065 words in vocabulary. The experiments are executed on the RNNLMs with 16 hidden units.}\label{fig:small_samples_train}\vspace{-0.0cm}
\end{figure*}

\section{Subnet Update With Approximated RMSProp} \label{app:subnet update}
RMSProp \citep{Hinton12} is an adaptive learning rate method that has found much success in practice. Instead of using a single learning rate to all the model parameters in $\Omega$, RMSProp dedicates a learning rate for each model parameter and normalizes the gradient by an exponential moving average of the magnitude of the gradient:
\begin{align}
  v_t = \beta v_{t-1} + (1 - \beta)(\nabla J)^2
  \label{eq:rmsprop velocity}
\end{align}
where $\beta\in(0,1)$ denotes the decay rate. The model update at time step $t$ is then given by
\begin{align}
  \theta_t=\theta_{t-1}+\epsilon\frac{\nabla J(\theta_{t-1})}{\sqrt{v_t+\lambda}}
  \label{eq:rmsprop update}
\end{align}
where $\epsilon$ is the learning rate and $\lambda$ is a damping factor, e.g., $\lambda=10^{-6}$. While RMSProp is one of the most effective learning rate scheduling techniques, it requires a large amount of memory to store per-parameter $v_t$ in addition to model parameter $\Omega$ and their gradients.

It is expensive to access and update large models with billions of parameters. Fortunately, due to the 1-of-$V$ encoding at input layer and the BlackOut sampling at output layer, the model update on $W_{in}$ and $W_{out}$ is sparse, e.g., only the model parameters corresponding to input/output words and the samples in $S_K$ are to be updated.\footnote{The parameter update on $W_r$ is still dense, but its size is several orders of magnitude smaller than those of $W_{in}$ and $W_{out}$.} For Eq.~\ref{eq:rmsprop velocity}, however, even a model parameter is not involved in the current training, its $v_t$ value still needs to be updated by $v_t=\beta v_{t-1}$ since its $(\nabla J)^2=0$. Ignoring this update has detrimental effect on the predictive performance; in our experiments, we observed $5-10$ point perplexity loss if we ignore this update completely. 

We resort to an approximation to $v_t=\beta v_{t-1}$. Given $p_u(w)$ is the probability of a word $w$ being selected for update, the number of time steps elapsed when it is successfully selected follows a geometric distribution with a success rate $p_u(w)$, whose mean value is $1/p_u(w)$. Assume that an input/output word is selected according to the unigram distribution $p_{uni}(w)$ and the samples in $S_K$ are drawn from $Q_\alpha(w)$, Eq.~\ref{eq:rmsprop velocity} can be approximated by
\begin{align}
  v_t \approx \beta^{1/p_u} v_{t-n} + (1 - \beta)(\nabla J)^2
  \label{eq:rmsprop velocity2}
\end{align}
with
\begin{align}
  \label{eq:success rate}
  p_u(w) =
  \begin{cases}
    p_{uni}(w)\times B\times T & \text{ for word $w$ at input layer}  \\
    p_{uni}(w)\times B\times T+Q_\alpha(w)\times K\times T & \text{ for word $w$ at output layer},
  \end{cases}
\end{align}
where $B$ is the mini-batch size and $T$ is the BPTT block size. Now we can only update the model parameters, typically a tiny fraction of $\Omega$, that are \textit{really} involved in the current training, and thus speed up the RNNLM training further.

\end{document}